\pdfoutput=1

\documentclass[11pt]{article}

\usepackage[]{acl}

\usepackage{times}
\usepackage{latexsym}

\usepackage[T1]{fontenc}

\usepackage[utf8]{inputenc}

\usepackage{amssymb}
\usepackage{amsmath}

\usepackage{physics}

\usepackage{graphicx}

\usepackage{amsthm}
\newtheorem{theorem}{Theorem}[section]
\newtheorem{definition}{Definition}[section]

\title{How reparametrization trick broke differentially-private text representation learning}

\author{Ivan Habernal \\
	Trustworthy Human Language Technologies  \\
	Department of Computer Science \\
	Technical University of Darmstadt \\
	\texttt{ivan.habernal@tu-darmstadt.de} \\
	\url{www.trusthlt.org}  \\}

\begin{document}

\onecolumn
\noindent \textbf{How reparametrization trick broke differentially-private text representation learning}

\medskip
\noindent Ivan Habernal

\bigskip
This is a \textbf{pre-print non-final version} of the article accepted for publication at the \emph{60th Annual Meeting of the Association for Computational Linguistics (ACL 2022)}. The final official version will be published on the ACL Anthology website in May 2022: \url{https://aclanthology.org/}

\medskip
Please cite this pre-print version as follows.
\medskip

\begin{verbatim}
@InProceedings{Habernal.2022.ACL,
title = {How reparametrization trick broke differentially-private
text representation learning},
author = {Habernal, Ivan},
publisher = {Association for Computational Linguistics},
booktitle = {Proceedings of the 60th Annual Meeting of the
Association for Computational Linguistics},
pages = {(to appear)},
year = {2022},
address = {Dublin, Ireland},
url = {https://arxiv.org/abs/2202.12138},
}
\end{verbatim}
\twocolumn

\maketitle

\begin{abstract}
As privacy gains traction in the NLP community, researchers have started adopting various approaches to privacy-preserving methods. One of the favorite privacy frameworks, differential privacy (DP), is perhaps the most compelling thanks to its fundamental theoretical guarantees. Despite the apparent simplicity of the general concept of differential privacy, it seems non-trivial to get it right when applying it to NLP.
In this short paper, we formally analyze several recent NLP papers proposing text representation learning using DPText \citep{Beigi.et.al.2019.arXiv,Beigi.et.al.2019.HT,Alnasser.et.al.2021,Beigi.et.al.2021.Patent} and reveal their false claims of being differentially private.
Furthermore, we also show a simple yet general empirical sanity check to determine whether a given implementation of a DP mechanism almost certainly violates the privacy loss guarantees.
Our main goal is to raise awareness and help the community understand potential pitfalls of applying differential privacy to text representation learning.
\end{abstract}

\section{Introduction}

Differential privacy (DP), a formal mathematical treatment of privacy protection, is making its way to NLP \citep{Igamberdiev.Habernal.2021.arXiv,Senge.et.al.2021.arXiv}. Unlike other approaches to protect privacy of individuals' text documents, such as redacting named entities \citep{Lison.et.al.2021.ACL} or learning text representation with a GAN attacker \citep{Li.et.al.2018.ACLShort}, DP has the advantage of \emph{quantifying} and \emph{guaranteeing} how much privacy can be lost in the worst case. However, as \citet{Habernal.2021.EMNLP} showed, adapting DP mechanisms to NLP properly is a non-trivial task.

Representation learning with protecting privacy in an end-to-end fashion has been recently proposed in DPText  \citep{Beigi.et.al.2019.HT,Beigi.et.al.2019.arXiv,Alnasser.et.al.2021}. DPText consists of an auto-encoder for text representation, a differential-privacy-based noise adder, and private attribute discriminators, among others. The latent text representation is claimed to be differentially private and thus can be shared with data consumers for a given down-stream task. Unlike using a pre-determined privacy budget $\varepsilon$, DPText takes $\varepsilon$ as a learnable parameter and utilizes the reparametrization trick \cite{Kingma.Welling.2014} for random sampling. However, the downstream task results look too good to be true for such low $\varepsilon$ values. We thus asked whether DPText is really differentially private.

This paper makes two important contributions to the community. First, we formally analyze the heart of DPText and prove that the employed reparametrization trick based on inverse continuous density function in DPText is wrong and the model violates the DP guarantees. This shows that extreme care should be taken when implementing DP algorithms in end-to-end differentiable deep neural networks. Second, we propose an empirical sanity check which simulates the actual privacy loss on a carefully crafted dataset and a reconstruction attack. This supports our theoretical analysis of non-privacy of DPText and also confirms previous findings of breaking privacy of another system ADePT.\footnote{ADePT is a text-to-text rewriting system claimed to be differentially private \citep{Krishna.et.al.2021.EACL} but has been found to be DP-violating  \citep{Habernal.2021.EMNLP}.}

\section{Differential privacy primer}

Suppose we have a dataset (database) where each element belongs to an individual, for example Alice, Bob, Charlie, up to $m$. Each person's entry, denoted with a generic variable $x$, could be an arbitrary object, but for simplicity consider it a real valued vector $x \in \mathbb{R}^k$. An important premise is that this vector contains some sensitive information we aim to protect, for example an income ($x \in \mathbb{R}$), a binary value whether or not the person has a certain disease ($x \in  \{0.0, 1.0\})$, or a dense representation from SentenceBERT containing the person's latest medical record ($x \in \mathbb{R}^k$). This dataset is held by someone we trust to protect the information, the trusted curator.\footnote{This is \emph{centralized DP}, as opposed to \emph{local-DP} where no such trusted curator exists.}

This dataset is a set from which we can create $2^m$ subsets, for instance $X_1 = \{\textrm{Alice}\}$, $X_2 = \{\textrm{Alice}, \textrm{Bob}\}$, etc. All these subsets form a \emph{universe} $\mathcal{X}$, that is $X_1, X_2, \dots \in \mathcal{X}$, and each of them is also called (a bit ambiguously) a dataset.

\begin{definition}
Any two datasets $X, X' \in \mathcal{X}$ are called \emph{neighboring}, if they differ in one person.
\end{definition}

For example, $X = \{\textrm{Alice}\}, X' = \{\textrm{Bob}\}$ or $X = \{\textrm{Alice}, \textrm{Bob}\}, X' = \{\textrm{Bob}\}$ are neighboring, while $X = \{\textrm{Alice}\}, X' = \{\textrm{Alice, Bob, Charlie}\}$ are not.

\begin{definition}
\emph{Numeric query} is any function $f$ applied to a dataset $X$ and outputting a real-valued vector, formally $f : X \to \mathbb{R}^k$.
\end{definition}

For example, numeric queries might return an average income ($f \to \mathbb{R}$), number of persons in the database ($f \to \mathbb{R})$, or a textual summary of medical records of all persons in the database represented as a dense vector ($f \to \mathbb{R}^k$). The query is simply something we want to learn from the dataset. A query might be also an identity function that just `copies' the input, e.g., $f(X = \{(1, 0)\}) \to (1, 0)$ for a real-valued dataset $X = \{(1, 0)\}$.

An attacker who knows everything about Bob, Charlie, and others would be able to reveal Alice's private information by querying the dataset and combining it with what they know already. Differentially private algorithm (or mechanism) $\mathcal{M}(X; f)$  thus randomly modifies the query output in order to minimize and quantify such attacks.
\citet{Smith.Ullman.2021} formulate the principle of differential privacy as follows:
\emph{``No matter what they know ahead of time, an attacker seeing the output of a differentially private algorithm would draw (almost) the same conclusions about Alice
whether or not her data were used.''
}

Let a DP-mechanism $\mathcal{M}(X; f)$ have an arbitrary range $\mathcal{R}$ (a generalization of our case of numeric queries, for which we would have $\mathcal{R} = \mathbb{R}^k$). 
Differential privacy is then defined as

\begin{equation}
\label{eq:dp-bayesian-def}
\frac{\Pr(X | \mathcal{M}(X; f) = z)}{\Pr(X' | \mathcal{M}(X; f) = z)}
\leq
\exp (\varepsilon) \cdot \frac{\Pr(X)}{\Pr(X')}
\end{equation}

for all neighboring datasets $X, X'$ and all $z \in \mathcal{R}$, where $\Pr(X)$ and $\Pr(X')$ is our prior knowledge of $X$ and $X'$. In words, our posterior knowledge of $X$ or $X'$ after observing $z$ can only grow by factor $\exp(\varepsilon)$ \citep{Mironov.2017.CSF}, where $\varepsilon$ is a \emph{privacy budget} \citep{Dwork.Roth.2013}.\footnote{In this paper, we will use the basic form of DP, that is $(\varepsilon, 0)$-DP. There are various other (typically more `relaxed') variants of DP, such $(\varepsilon, \delta)$-DP, but they are not relevant to the current paper as DPText also claims $(\varepsilon, 0)$-DP.}

\section{Analysis of DPText}

In the heart of the model, DPText relies on the standard Laplace mechanism which takes a real-valued vector and perturbs each element by a random draw from the Laplace distribution.

Formally, let $\mathbf{z}$ be a real-valued $d$-dimensional vector. Then the Laplace mechanism outputs a vector $\tilde{\mathbf{z}}$ such that for each index $i = 1, \dots, d$
\begin{equation}
\tilde{z}_i = z_i + s_i
\end{equation}
where each $s_i$ is drawn independently from a Laplace distribution with zero mean and scale $b$ that is proportional to the $\ell_1$ sensitivity $\Delta$ and the privacy budget $\varepsilon$, namely

\begin{equation}
s_i \sim \mathrm{Lap}\left( \mu = 0; b = \frac{\Delta}{\varepsilon} \right)
\end{equation}

The Laplace mechanism satisfies differential privacy \citep{Dwork.Roth.2013}.

\subsection{Reparametrization trick and inverse CDF sampling}

DPText employs the variational autoencoder architecture in order to directly optimize the amount of noise added in the latent layer parametrized by $\varepsilon$. In other words, the scale of the Laplace distribution becomes a trainable parameter of the network. As directly sampling from a distribution is known to be problematic for end-to-end differentiable deep networks, DPText borrows the reparametrization trick from \citet{Kingma.Welling.2014}.

In a nutshell, the reparametrization trick decouples drawing a random sample from a desired distribution (such as Exponential, Laplace, or Gaussian) into two steps: First draw a value from another distribution (such as Uniform), and then transform it using a particular function, mainly the inverse continuous density function (CDF).

As a matter of fact, sampling using the inverse CDF is a well-known and widely used method \citep{Devroye.1986,Ross.2012}  and forms the backbone of probability distribution generators in many popular frameworks.

\subsection{Inverse CDF of Laplace distribution}
\label{sec:inverse-cdf}

The inverse cumulative distribution function of Laplace distribution $\mathrm{Lap}(\mu; b)$ is:

\begin{equation}
\label{eq:cdf.ver1}
F^{-1}(u) = \mu - b\,\mathrm{sgn}(u-0.5)\,\ln(1 - 2|u-0.5|)
\end{equation}

where $u \sim \mathrm{Uni}(0, 1)$ is drawn from a standard uniform distribution \citep[p.~210]{Sugiyama.2016}, \citep[p.~303]{Nahmias.Olsen.2015}.
An equivalent expression without the $\mathrm{sgn}$ and absolute functions is derived, e.g., by \citet[p.~166]{Li.et.al.2018.Healthcare} as
\begin{equation}
F^{-1}(u) =
\begin{cases}
b \ln (2 u) + \mu       & \quad \text{if } u < 0.5\\
\mu - b \ln (2(1 - u))  & \quad \text{if } u \geq 0.5
\end{cases}
\end{equation}
where again $u \sim \mathrm{Uni}(0, 1)$.\footnote{This implementation is used in \texttt{numpy}, see \url{https://github.com/numpy/numpy/blob/maintenance/1.21.x/numpy/random/src/distributions/distributions.c\#L469}}

An alternative sampling strategy, as shown, e.g., by \citet[p.~62]{Al-Shuhail.Al-Dossary.2020}, assumes that the random variable is drawn from a shifted, zero-centered uniform distribution
\begin{equation}
v \sim \mathrm{Uni}\left( -0.5, +0.5 \right)
\end{equation}
and transformed through the following function
\begin{equation}
\label{eq:cdf.ver2}
F^{-1}(v) = \mu - b\ \mathrm{sgn}(v) \ln (1 - 2 |v|)
\end{equation}

While both (\ref{eq:cdf.ver1}) and (\ref{eq:cdf.ver2}) generate samples from $\mathrm{Lap}(\mu; b)$, note the substantial difference between $u$ and $v$, since each is drawn from a different uniform distribution.

\subsection{Proofs of DPText violating DP}

According to Eq.~3 in \citep{Alnasser.et.al.2021}, Eq.~9 in \citep{Beigi.et.al.2019.arXiv} which is an extended version of \citep{Beigi.et.al.2019.HT}, in Eq.~14 in \citep{Beigi.et.al.2021.Patent}, and personal communication to confirm the formulas, the main claim of DPText is as follows (rephrased):

\begin{quotation}
DPText utilizes the Laplace mechanism, which is DP \citep{Dwork.Roth.2013}. It implements the mechanism as follows:
Sampling a value from standard uniform
\begin{equation}
\label{eq:uni0}
v \sim \mathrm{Uni}(0, 1)
\end{equation}
and transforming using
\begin{equation}
\label{eq:cdf-wrong}
F^{-1}(v) = \mu - b\ \mathrm{sgn}(v) \ln (1 - 2 |v|)
\end{equation}
is equivalent to sampling noise from $\mathrm{Lap}(b)$.
\end{quotation}
This claim is unfortunately false, as it mixes up both approaches introduced in Sec.~\ref{sec:inverse-cdf}.
As a consequence, the Laplace mechanism using such sampling is not DP, which we will first prove formally.

\begin{theorem}
\label{theorem:dptext.always.positive}
Sampling using inverse CDF as in DPText using (\ref{eq:uni0}) and (\ref{eq:cdf-wrong}) does not produce Laplace distribution.
\end{theorem}

\begin{proof}
	We will rely on the standard proof of sampling from inverse CDF (see Appendix~\ref{sec:proof.cdf}). The essential step of that proof is that the CDF is increasing on the support of the uniform distribution, that is on $[0, 1]$. However, $F^{-1}$ as used in (\ref{eq:cdf-wrong}) is increasing only on interval $[0, 0.5]$. For $v \geq 0.5$, we get negative argument to $\ln$ which yields a complex function, whose real part is even decreasing. Therefore (\ref{eq:cdf-wrong}) is not CDF of any probability distribution, if used with $\text{Uni}(0, 1)$.
\end{proof}

As a consequence, the output $\ln(v \leq 0)$ arbitrarily depends on the particular implementation. In \texttt{numpy}, it is \texttt{NaN} with a warning only. Therefore this function samples only positive or \texttt{NaN} numbers. Since DPText sources are not publicly available, we can only assume that \texttt{NaN} numbers are either replaced by zero, or the sampling proceeds as long as the desired number of samples is reached (discarding \texttt{NaN}s). In either case, no negative values can be obtained.
See Fig.~\ref{fig:distributions} in the Appendix for various Laplace-based distributions sampled with different techniques including possible distributions sampled in DPText.

\begin{theorem}
\label{theorem:proof.dptext}
DPText with private mechanism based on (\ref{eq:uni0}) and (\ref{eq:cdf-wrong}) fails to guarantee differential privacy.
\end{theorem}

\begin{proof}
	We rely on the standard proof of the Laplace mechanism as shown, e.g, by \citet{Habernal.2021.EMNLP}. Let $X = 0$ and $X' = 1$ be two neighboring datasets, and the query $f$ being the identity query, such that it outputs simply the value of $X$. Let the DPText mechanism $\mathcal{M}(X; f)$ outputs a particular value $z$.
	
	In order to being differentially private, mechanism $\mathcal{M}(X; f)$ has to fulfill the following bound of the privacy loss:
	\begin{equation}
	\label{eq:eps0dp1}
	\abs{
		\frac{
			\Pr(\mathcal{M}(X) = z)
		}{
			\Pr(\mathcal{M}(X') = z)
		}
	}
	\leq
	\exp(\varepsilon)
	\end{equation}
	for all neighboring datasets $X, X' \in \mathcal{X}$ and all outputs $z \in \mathcal{R}$ from the range of $\mathcal{M}$, provided that our priors over $X$ and $X'$ are uniform (cf.~Eq.~\ref{eq:dp-bayesian-def}).
	
	Fix $z = 0.1$. Then $\Pr(\mathcal{M}(X) = 0.1)$ will have a positive probability (recall it takes the query output $f(X = 0) = 0$ and adds a random number drawn from the probability distribution, which is always positive as shown in Theorem~\ref{theorem:dptext.always.positive}.) However $\Pr(\mathcal{M}(X') = 0.1)$ will be zero, as the query output  $f(X' = 1) = 1$ will be added again only a positive random number and thus never be less then $1$. By plugging this into (\ref{eq:eps0dp1}), we obtain
	\begin{equation}
	\label{eq:eps0dp2}
	\abs{
		\frac{
			\Pr(\mathcal{M}(X) = 0.1)
		}{
			\Pr(\mathcal{M}(X') = 0.1)
		}
	}
	= \frac{\Pr > 0}{\Pr = 0}
	\nleq
	\exp(\varepsilon)
	\end{equation}
	which results in an infinity privacy loss and violates differential privacy.
\end{proof}

\begin{figure*}[h!]
	\begin{center}
		\includegraphics[height=12.9em]{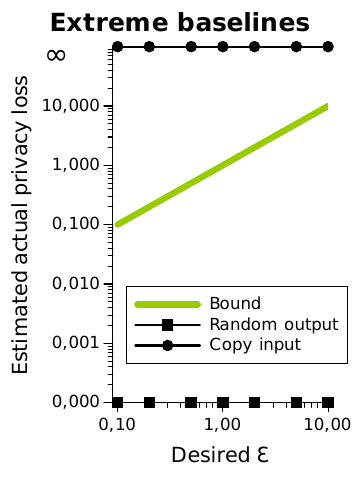}
		\includegraphics[height=12.9em]{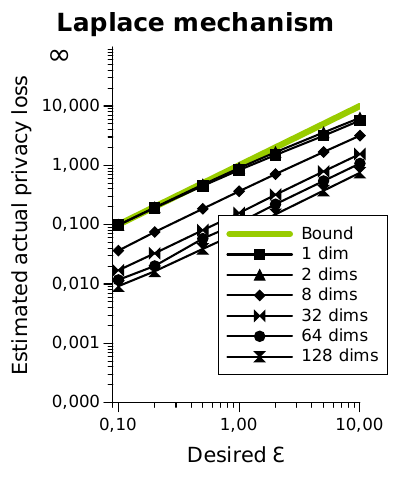}
		\includegraphics[height=12.9em]{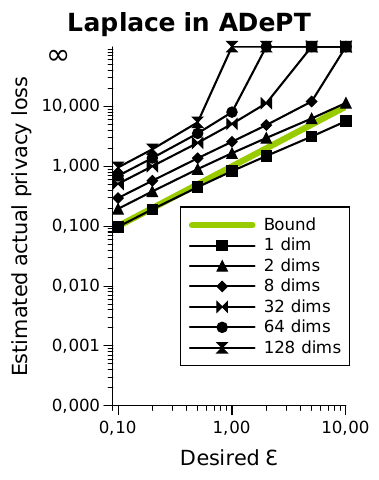}
		\includegraphics[height=12.9em]{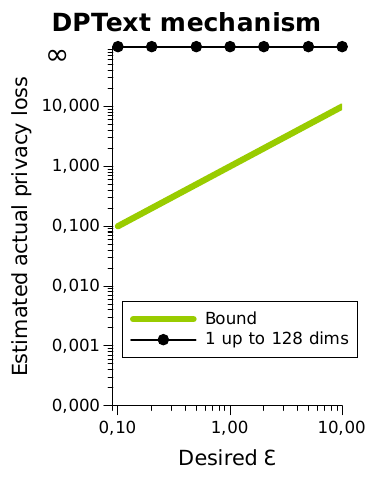}
	\end{center}
	\caption{\label{fig:results}Area under the green line: Our attack does not reveal more than allowed by the desired privacy budget. Note that it does not guarantee DP, the reconstruction attack might be just weak. Area above the green line: The algorithm almost certainly violates DP as our attack caused bigger privacy loss than allowed by $\varepsilon$. \emph{Extreme baselines} show two extreme scenarios, as \emph{random output} is absolutely private (but provides zero utility) and \emph{copy input} provides maximal utility but no privacy by revealing the data in full.}
\end{figure*}

\section{Empirical sanity check algorithm}

It is impossible to empirically verify that a given DP-mechanism implementation is actually DP \citep{Ding.et.al.2018.SIGSAC}. However, it is possible to detect a DP-violating mechanism with a fair degree of certainty. We propose a general sanity check applicable to any real-valued DP mechanism, such as the Laplace mechanism, DPText, or any other.\footnote{Some related works along these lines also utilize statistical analysis of the source code written in a C-like language \citep{Wang.et.al.2020.SIGSAC}.}

We start by constructing two neighboring datasets $X$ (Alice) and $X'$ (Bob) such that $X = (0, \dots, 0_n)$ consists of $n$ zeros and $X' = (1, \dots, 1_n)$ consists of $n$ ones. The dimensionality $n \in \{1, 2, \dots\}$ is a hyperparameter of the experiment. 
We employ a synthetic data release mechanism (also called local DP). The mechanism takes $X$ or $X'$ and outputs its privatized version of the same dimensionality $n$, so that the zeros or ones are `noisified' real numbers. The query sensitivity $\Delta$ is $n$.\footnote{See \citep{Dwork.Roth.2013} for $\ell_1$-sensitivity definition.}

Thanks to the post-processing lemma, any post-processing of DP output remains DP. We can thus turn the output real vector back to all zeros or all ones, simply by rounding to closest $0$ or $1$ and applying majority voting. This process is in fact our reconstruction attack: given a privatized vector, we try to guess what the original values were, either all zeros or all ones.

What our attacker is doing, and what DP protects, is that if Alice gives us her privatized data, we cannot tell whether her private values were all zeros or all ones (up to a given factor); the same for Bob.

By definition (\ref{eq:dp-bayesian-def}) and having no prior knowledge over $X$ and $X'$ apart from the fact that the values are correlated, our attacker cannot exceed the guaranteed privacy loss $\exp(\varepsilon)$:
\begin{equation}
\frac{\Pr(X | \mathcal{M}(X; f) = z)}{\Pr(X' | \mathcal{M}(X; f) = z)}
\leq
\exp (\varepsilon)
\end{equation}

We can estimate the conditional probability $\Pr(X | \mathcal{M}(X; f) = z)$ using maximum likelihood estimation (MLE) simply as our attacker's precision: How many times the attacker reconstructed true $X$ values given the observed privatized vector. We can do the same for estimating the conditional probability of $X'$.
In particular, we repeatedly run each DP mechanism over $X$ and $X'$ 10 million times each, which gives very precise MLE estimates even for small $\varepsilon$.\footnote{For example, we repeated the full experiment on ADePT ($n = 2$, $\varepsilon = 0.1$) 100 times which results in standard deviation $0.0008$ from the mean value $0.195$. Better MLE precision can be simply obtained by increasing the 10 million repeats per experiment. Source codes available at \url{https://github.com/trusthlt/acl2022-reparametrization-trick-broke-differential-privacy}}

\section{Results and discussion}

For the sake of completeness, we implemented two extreme baselines: One that simply copies input (no privacy) and other one completely random regardless of the input (maximum privacy); these are shown in Figure~\ref{fig:results} left. The vanilla Laplace mechanism behaves as expected; all empirical losses for all dimensions (1 up to 128) are bounded by $\varepsilon$. We re-implemented the Laplace mechanism from ADePT \citep{Krishna.et.al.2021.EACL} that, due to wrong sensitivity, has been shown theoretically as DP-violating \citep{Habernal.2021.EMNLP}. We empirically confirm that ADePT suffered from the curse of dimensionality as the privacy loss explodes for larger dimensions. The last panel confirms our previous theoretical DPText results, which (regardless of dimensionality) has infinite privacy loss.

Note that we constructed the dataset carefully as two neighboring multidimensional correlated data that are as distant from each other as possible in the $(0,1)^n$ space. However, DP must guarantee privacy for any datapoints, even the worst case scenario, as shown by the correct Laplace mechanism.

\section{Conclusion}

We formally proved that DPText \citep{Beigi.et.al.2019.HT,Beigi.et.al.2019.arXiv,Alnasser.et.al.2021,Beigi.et.al.2021.Patent} is not differentially private due to wrong sampling in its reparametrization trick. We also proposed an empirical sanity check that confirmed our findings and can help to reveal potential errors in DP mechanism implementations for NLP.

\section{Ethics Statement}

We declare no conflict of interests with the authors of DPText, we do not even know them personally. The purpose of this paper is strictly scientific.

\section*{Acknowledgements}

The independent research group TrustHLT is supported by the Hessian Ministry of Higher Education, Research, Science and the Arts.
Thanks to Cecilia Liu, Haau-Sing Li, and the anonymous reviewers for their helpful feedback.
A special thanks to Condor airlines, whose greed to make passengers pay for everything resulted in the most productive transatlantic flights I've ever had.

\bibliography{bibliography}
\bibliographystyle{acl_natbib}

\appendix

\begin{figure*}[t!]
	\includegraphics[width=\linewidth]{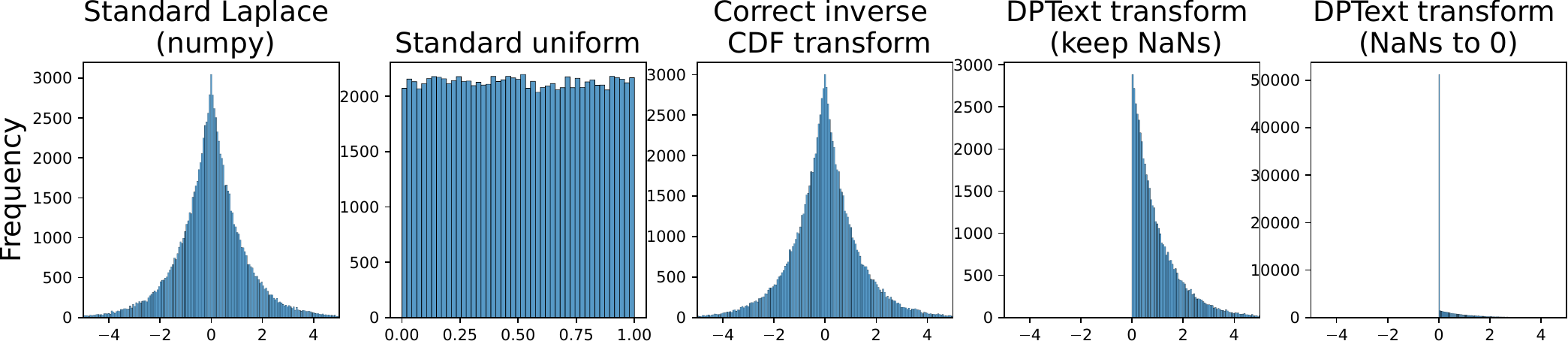}
	\caption{\label{fig:distributions} Comparing sampling strategies. Left: Sampling using vanilla \texttt{numpy} implementation. Second from the left: Uniform sample as basis for the following three inverse CDF transformations. Generated with 100k samples.}
\end{figure*}

\section{Proof of sampling from inverse CDF}
\label{sec:proof.cdf}

Important fact 1: A random variable $U$ is uniformly distributed on $[0, 1]$ if the following holds
\begin{equation}
\label{eq:uniform.density.lemma}
U \sim \mathrm{Uni}(0, 1) \iff \Pr(U \leq u) = u.
\end{equation}
Important fact 2: For any function $g(\cdot)$ with an inverse function $g^{-1}(\cdot)$, the following holds
\begin{equation}
\label{eq:inverse.inverse.identity}
g(g^{-1}(x)) = x; \quad g^{-1}(g(x)) = x.
\end{equation}
Important fact 3: For any increasing function $g(\cdot)$, we have by definition
\begin{equation}
\label{eq:increasing.equiv}
x \leq y \implies g(x) \leq g(y).
\end{equation}
We know that $\Pr(X \leq a)$ is a shortcut for probability of event $E_1$ defined using the set-builder notation as $E_1 = \{ s \in \Omega : X(s) \leq a \}$. Then by plugging (\ref{eq:increasing.equiv}) into the predicate of $E_1$, we obtain an equal set, namely event $E_2 = \{s \in \Omega : g(X(s)) \leq g(a)\}$, for which the probability must be the same. Therefore for any random variable $X$ and increasing function $g(\cdot)$ we have
\begin{equation}
\label{eq:same.probs}
\Pr(X \leq a) = \Pr(g(X) \leq g(a)).
\end{equation}

\begin{theorem}
\label{theorem:cdf}
Let $U$ be a uniform random variable on $[0, 1]$. Let $X$ be a continuous random variable with CDF (cumulative distribution function) $F(\cdot)$. Let $Y$ be defined such that $Y = F^{-1}(U)$. Then $Y$ has CDF $F(\cdot)$.
\end{theorem}

\begin{proof}

Function $F(\cdot)$ is the CDF of a continuous random variable $X$, and as a CDF its range is $[0, 1]$. Also, if $F(\cdot)$ is strictly increasing, it has a unique inverse function $F^{-1}(\cdot)$ defined on $[0, 1]$.

We defined $Y = F^{-1}(U)$, so consider
\begin{equation}
\Pr(Y \leq y) = \Pr(F^{-1}(U) \leq y).
\end{equation}
\label{eq:plugging.same.probs}
Since $F(\cdot)$ is increasing, using (\ref{eq:same.probs}) we get
\begin{equation}
\Pr(Y \leq y) = \Pr(F(F^{-1}(U)) \leq F(y)).
\end{equation}
Now plugging (\ref{eq:inverse.inverse.identity}) we obtain
\begin{equation}
\Pr(Y \leq y) = \Pr(U \leq F(y)),
\end{equation}
and finally by (\ref{eq:uniform.density.lemma})
\begin{equation}
\Pr(Y \leq y) = F(y).
\end{equation}

\end{proof}
For an overview of proofs of Theorem~\ref{theorem:cdf} see \citep{Angus.1994}.

\end{document}